\documentclass{article}

\usepackage{latexsym}
\usepackage{amssymb}
\usepackage{amsmath}
\usepackage{amsthm}
\usepackage{mathabx}
\usepackage{booktabs}
\usepackage{enumitem}
\usepackage{graphicx}
\usepackage{color}
\usepackage{scalerel}
\usepackage{soul}
\usepackage{mdframed}
\usepackage{enumitem}
\usepackage{numprint}
\usepackage{authblk}
\usepackage[colorlinks = true,
            linkcolor = blue,
            urlcolor  = blue,
            citecolor = blue,
            anchorcolor = blue]{hyperref}
\usepackage{cleveref}
\usepackage{natbib}
\mdfsetup{backgroundcolor=cyan}

\newtheorem{theorem}{Theorem}
\newtheorem{lemma}{Lemma}

\newtheorem{definition}{Definition}
\theoremstyle{definition}

\newtheorem{remark}{Remark}

\newcommand{\BibTeX}{B\kern-.05em{\sc i\kern-.025em b}\kern-.08em\TeX}
\DeclareMathOperator*{\argmax}{argmax}

\title{Towards Privacy-Aware Bayesian Networks: A Credal Approach}

\author[1,2]{Niccolò Rocchi}

\author[1,*]{Fabio Stella}

\author[3]{Cassio de Campos}

\affil[1]{University of Milano-Bicocca, Italy}
\affil[2]{Fondazione IRCCS Istituto Nazionale dei Tumori, Italy}
\affil[3]{Eindhoven University of Technology, the Netherlands}
\affil[*]{Corresponding author. Email: fabio.stella@unimib.it}

\date{}

\begin{document}

\maketitle

\begin{abstract}
    \textit{Bayesian networks} (BN) are versatile probabilistic graphical models that enable efficient knowledge representation and inference. These models have proven effective across diverse domains, including healthcare, bioinformatics, economics, law, and image processing. The structure and parameters of a BN can be obtained by domain experts or directly learned from available data. However, as privacy concerns escalate, it becomes increasingly critical for publicly released models to safeguard sensitive information in training data. Typically, released models do not prioritize privacy by design, and the issue equally affects BNs. In particular, \textit{tracing attacks} from adversaries can combine the released BN with auxiliary data to determine whether specific individuals belong to the data from which the BN was learned. The current approach to addressing this privacy issue involves introducing noise into the learned parameters. While this method offers robust protection against tracing attacks, it also significantly impacts the model's utility, in terms of both the significance and accuracy of the resulting inferences. Hence, high privacy may be attained, but at the cost of releasing a possibly ineffective model. This paper introduces \textit{credal networks} (CN) as a novel and practical solution for balancing the model's privacy and utility. Specifically, after adapting the notion of tracing attacks, we demonstrate that a CN enables the masking of the learned BN, thereby reducing the probability of successful tracing attacks. As CNs are obfuscated but not noisy versions of BNs, they can achieve meaningful inferences while safeguarding the privacy of the released model. Moreover, we identify key learning information that must be concealed to prevent attackers from recovering the BN underlying the released CN. Finally, we conduct a set of numerical experiments to analyze how privacy gains can be modulated by tuning the CN hyperparameters. Our results confirm that CNs provide a principled, practical, and effective approach towards the development of privacy-aware probabilistic graphical models.
\end{abstract}

\section{Introduction}

    Experts across various domains rely on \textit{artificial intelligence} (AI) systems to support their daily work and critical decision-making. 
    Unlike opaque AI methods, transparent models such as \textit{Bayesian networks} (BN) \cite{Koller_2010} inspire greater trust and encourage broader adoption, providing rich insights into the data generation process. 
    A BN consists of (i) a graph whose nodes are associated with random variables and whose edges encode conditional (in)dependencies among the variables, and (ii) a compact set of parameters that allow for representing the variables' joint probability distribution efficiently. 
    When interpreted in \textit{causal} terms \cite{Pearl_1995}, a BN becomes a powerful tool to model cause-effect relationships, essential for understanding how decisions propagate towards different outcomes \cite{Bernasconi_2024,Feuerriegel_2024}. 
    These reasons have fueled the widespread adoption of BNs in fields where statistical and causal reasoning are central, including oncology \cite{Grube_2022}, neuroscience \cite{Roy_2023}, and genetics \cite{Suter_2022}.
    Learning a BN involves two main tasks: recovering its graphical structure and estimating its parameters \cite{Scanagatta_2019}. These tasks can be accomplished using data alone, by leveraging the expertise of domain specialists, or through a combination of both methods \cite{Constantinou_2023}.
    Even more, multi-center studies are beneficial for overcoming the limitations of individual data centers, such as data scarcity and biases \cite{Zanga_2023}.
    However, strict privacy regulations often prohibit the direct sharing of individual data across institutions. 
    As a result, \textit{federated learning} \cite{Zhang_2021} approaches emerged. A typical scenario involves all centers learning their own BNs locally and then sharing them with a global aggregator to develop a \textit{global} model \cite{Torrijos_2025}.
    
    Releasing the learned BN can still leave sensitive information vulnerable \cite{Murakonda_2021}. Adversaries who combine the model with auxiliary data (for example, public population registries) can develop \textit{tracing attacks} \cite{Carlini_2022}, which determine whether specific individual records were included in the training data set, i.e., the data used to learn the BN.
    Existing privacy-preserving strategies for BNs typically rely on adding random noise to the network parameters. In this case, the obtained model is said to be \textit{sanitized}. Typical choices for the noise include the Laplacian random perturbation \cite{Zhang_2017}, which has been proven to ensure \textit{differential privacy} \cite{Dwork_2017}. 
    Although differential privacy provides robust protection against tracing attacks, perturbation techniques can lead to inconsistencies in the BN marginal distributions when not properly managed \cite{Barak_2007}. Furthermore, the magnitude of the injected noise can be substantial, yet unnoticeable in the sanitized BN. Consequently, the released model may not accurately represent the underlying data, leading to potentially unreliable inferences \cite{Binkyte_2024}.
    Murakonda et al. \cite{Murakonda_2021} established an approximation of the effectiveness of tracing attacks on BNs. They demonstrated that, for any given error rate, the attacker's power is dependent on the network complexity and the sample size.
    Nevertheless, finding a practical way to \textit{balance} the model's privacy and utility remains a challenging, short-blanket problem that is not adequately addressed in the BN literature \cite{Rigaki_2023, Shao_2023}. 

    In this work, we build on this gap to introduce a novel approach to the privacy-utility balance problem in BNs, inspired by the relationship between privacy and Walley's \textit{imprecise probability} theory \cite{Walley_2000} pointed out in the recent research \cite{Bailie_2024,Li_2022}.
    In particular, we leverage \textit{credal networks} (CN) \cite{Cozman_2000}, an enhanced version of BNs that incorporates epistemic uncertainty into their parameters.
    We theoretically and empirically show that releasing a CN as the \textit{masked} version of a BN provides a robust defense against tracing attacks. 
    The parameters of a CN are represented as intervals, or more generally as sets, which enclose the point estimates of the ground-truth BN parameters. Therefore, they can yield \textit{guaranteed} inferential bounds using diverse (exact or approximate) algorithmic approaches \cite{ANTONUCCI201525,deCampos_2007,MAUA2020133}. 
    While the inferential bounds may sometimes be quite broad, potentially leading to ``I don't know'' conclusions, they still ensure correct responses to queries as opposed to noisy BNs. Consequently, CNs are valuable for offering a more usable and explainable form of privacy \cite{Bullek_2017}.
    Moreover, by controlling the level of uncertainty injected into the BN, a CN can practically balance the privacy level and the uncertainty of inferences. 
    
    The paper first introduces the concepts and notation of BNs and CNs (\Cref{sec:preliminaries}), followed by a formal definition of a tracing attack against a BN, along with known theoretical results (\Cref{sec:prob_stat}).    
    The notion of tracing attack is then extended to CNs while proving its consistency (\Cref{sec:attack_CN}). \Cref{sec:pr_g} proves, in the sample size limit, that releasing a CN ensures equal or higher privacy than releasing a BN (\Cref{th:power}). Finally, \Cref{sec:exp} presents and discusses an extensive set of numerical experiments that simulate tracing attacks on different networks. Results suggest that the privacy gained by a CN may even exceed the theoretical bound established in \Cref{th:power}.

\section{Probabilistic Graphical Models}
\label{sec:preliminaries}

    \begin{table}[t]
        \caption{Notation adopted throughout the paper; ``distrib.'' stands for distribution, while ``param.'' for parameters.\label{tab:notation}}
        \centering
        \begin{tabular}{l|l}
            \hline
            \textbf{Symbol} & \textbf{Meaning} \\
            \hline
            $|\mathcal{S}|$                         & The cardinality of a set $\mathcal{S}$.                  \\
            $\mathbf{X}$                            & Discrete random variables having support $\text{supp}(\mathbf{X})$.                  \\
            $P(\mathbf{X}\mid\theta_0)$             & Multinomial distrib. over $\text{supp}(\mathbf{X})$, parametrized by $\theta_0$.                                  \\
            $p(\cdot\mid\theta_0)$                  & Probability mass function, $p(\mathbf{x}\mid\theta_0) = P(\mathbf{X}=\mathbf{x}\mid\theta_0)$.              \\
            $(\mathcal{G}, \theta_0)$               & BN underlying $P(\mathbf{X}\mid\theta_0)$.                               \\
            $\Theta$                                & Space of param. that factorize according to $\mathcal{G}$, $\theta_0\in\Theta$.      \\

            $\mathcal{P}$                           & General population over variables $\mathbf{X}$.                                                     \\
            $\mathcal{R}, \mathcal{T}$              & Reference and target populations, $\mathcal{R}, \mathcal{T}\subset \mathcal{P}$.          \\
            $\widehat{K}_{\mathcal{T}}(\mathbf{X})$ & Estimated credal set from data $\mathcal{T}$, $\widehat{K}_{\mathcal{T}}(\mathbf{X}) \subset \Theta$. \\
            $\widehat{\theta}_{\mathcal{T}}$        & MLE of $\theta_0$ from data $\mathcal{T}$ over space $\Theta$.                            \\
            $\widehat{\theta}^{K}_{\mathcal{T}}$    & MLE of $\theta_0$ from data $\mathcal{T}$ over space $\widehat{K}_{\mathcal{T}}(\mathbf{X})\cap\Theta$.      \\
            $L(\mathbf{x}, \theta)$                          & Log-likelihood ratio (LLR) function over $\text{supp}(\mathbf{X})\times\Theta$. 
            \\
            $L(\mathbf{x})$ & LLR function where 2nd argument is $\widehat{\theta}_{\mathcal{T}}$, i.e. $L(\mathbf{x},\widehat{\theta}_{\mathcal{T}})$.\\
            \hline
        \end{tabular}
    \end{table}

    \textit{Probabilistic graphical models} (PGM) \cite{Koller_2010} are practical tools for representing and reasoning with multidimensional probability distributions. This article focuses on two PGMs: BNs and CNs. The notation adopted throughout the paper is summarized in \Cref{tab:notation}.
    
    BNs and CNs both rely on the structure provided by a graph. A \textit{graph} $\mathcal{G}$ is a pair $\mathcal{G}=(\mathbf{X},\mathbf{E})$, where $\mathbf{X}$ is the set of \textit{nodes} and $\mathcal{E}$ is a set of edges between pairs of nodes. If $X\rightarrow Y$ in $\mathcal{G}$ then $X$ is said to be a \textit{parent} of $Y$; $\Pi_Y$ denotes the set of parents of $Y$.
    A \textit{directed and acyclic graph} (DAG) is a graph where all edges are directed ($X\rightarrow Y$) and there is no sequence of nodes $ \ X\rightarrow \cdots \rightarrow Y$ so that $X = Y$. 
    In the following, we consider DAGs $\mathcal{G}$ whose nodes are associated with random variables. Namely, $X\in\mathbf{X}$ represents both a node in $\mathcal{G}$ and a random variable with given marginal distribution $P(X)$. Assuming to know $P(X)$, we denote by $\text{supp}(X)=\{x:P(X = x)>0\}$ its support. Similarly, $\text{supp}(\mathbf{X})$ is the Cartesian product $\bigtimes_{X\in\mathbf{X}}\text{supp}(X)$. All supports are assumed to be fixed and known.
    Our work focuses on \textit{categorical} random variables, i.e., a set $\mathbf{X}$ so that $|\text{supp}(\mathbf{X})|<\infty$. In this case, the vector of $\mathbf{X}$ has a joint probability distribution $P(\mathbf{X}\mid \theta_0)$ being a $\text{Multinomial}(\theta_0)$ over $\text{supp}(\mathbf{X})$, for some set of parameters $\theta_0$.

\subsection{Bayesian Networks}
    
    BNs \cite{Koller_2010} efficiently represent high-dimensional probability distributions, facilitating effective knowledge representation and inference.

    \begin{definition}[BN]
        A BN is a pair $(\mathcal{G}, \theta_0)$ where $\mathcal{G}$ is a DAG over random variables $\mathbf{X}$ and $\theta_0$ a set of parameters that factorizes accordingly to $\mathcal{G}$, namely:
        \begin{equation}
            \label{eq:bn_factorization}
            P(\mathbf{X}| \mathcal{G}, \theta_0) = \prod_{X \in \mathbf{X}} P(X |\Pi_X, \theta_{X})\, ,
        \end{equation}
    where $\theta_{X}$, which is part of $\theta_0$, are parameters defining the \textit{local} distribution of $X$.
    \end{definition}
    
    As we restrict ourselves to categorical random variables, $p(x \mid \Pi_X = \pi_X, \theta_{X})$ follows a multinomial for any $\pi_X \in \text{supp}(\Pi_X)$. Hence, the \textit{complexity} of a BN is determined by the total number of configurations of $\mathbf{X}$:

    \begin{equation}
        \label{eq:compl}
        C(\mathcal{G}) = \sum_{X\in\mathbf{X}} |\text{supp}(\Pi_X)|\cdot (|\text{supp}(X)| - 1)\, .
    \end{equation}

    A BN can be learned from data alone, and we focus on this setting. However, domain experts can also build BNs by hand \cite{Xiao_2007}, but this process may be highly time-consuming, especially as the number of nodes and parameters increases. Hybrid approaches leverage both data and experts' knowledge to mitigate human-level and data-level biases \cite{Constantinou_2023}. Those approaches proved effective in various scenarios, especially when causal relationships are sought \cite{Zanga_2022}.
    
    Learning a BN involves two main tasks: learning the DAG (the \textit{structure learning} task), and estimating its parameters (the \textit{parameter learning} task). Denoting the available data by $\mathcal{D}$, we can express the whole learning task in a Bayesian fashion:
    \begin{equation}
        \underbrace{P(\mathcal{G}, \theta \mid \mathcal{D})}_{\text{BN learning}} =
        \underbrace{P(\mathcal{G} \mid \mathcal{D})}_{\text{Structure learning}}
        \cdot
        \underbrace{P(\theta \mid \mathcal{G}, \mathcal{D})}_{\text{Parameter learning}}\, ,
    \end{equation}
    where we seek the optimum of the left-hand-side term by exploring the spaces of DAGs and parameters.

    There are two primary approaches for the structure learning task \cite{Kitson_2023, Scanagatta_2019}. The first, known as the \textit{score-based} approach, aims to optimize $P(\mathcal{G} \mid \mathcal{D})$ by maximizing the graph likelihood $P(\mathcal{D} \mid \mathcal{G})$, which is treated as the graph \textit{score}. The second approach, called \textit{constraint-based}, iteratively narrows down the space of all DAGs that are consistent with $\mathcal{D}$. 

    In this paper, we assume that the structure learning task has been accomplished, thus $\mathcal{G}$ is known. Consequently, we omit the condition $(\cdot \mid \mathcal{G})$ in all subsequent formulas for convenience.
    The parameter learning task can be reformulated as follows:

    \begin{equation}
        \underbrace{P(\theta \mid \mathcal{D})}_{\text{Posterior}} \propto \underbrace{P(\theta)}_{\text{Prior}} \cdot \underbrace{P(\mathcal{D} \mid \theta)}_{\text{Likelihood}}\, .
    \end{equation}

    Consider $\Theta$ as the space of all possible sets of parameters that factorize accordingly to $\mathcal{G}$, namely those for which it holds Eq.~\eqref{eq:bn_factorization}. There are two common approaches to parameter learning. The first one is the \textit{Maximum likelihood estimation} (MLE), which assumes a uniform prior and aims at maximizing the likelihood of the parameters, namely at finding $\widehat{\theta}_{\mathcal{D}} = \argmax_{\theta \in \Theta} P(\mathcal{D} \mid \theta)$.
    The second, more robust approach is known as \textit{Bayesian posterior estimation} \cite{Koller_2010}. The prior $P(\theta)$ is assumed to follow a Dirichlet distribution, which is the conjugate prior of a multinomial likelihood. Hence, the posterior follows a Dirichlet distribution, which can then be maximized (more details are provided in \Cref{sec:cn}).

\subsection{Credal Networks}
    \label{sec:cn}

    CNs \cite{Cozman_2000} are models based on Walley's imprecise probability theory \cite{Walley_2000} and introduce epistemic uncertainty in the BNs' parameters.

    \begin{definition}[Credal set]
        A \textit{credal set} for a random variable $X$ is a closed and convex set of probability distributions over $\textnormal{supp}(X)$, denoted by $K(X)$. The \textit{joint credal set} over $\mathbf{X}$ is defined similarly over $\textnormal{supp}(\mathbf{X})$.\footnote{Support here means the set of states that \textit{might} have positive probability.} 
    \end{definition}
     
    Being $X$ categorical, $K(X)$ can be thought of as a convex hull in $[0,1]^{|\text{supp}(X)|-1}$, which is the whole parameter space for $X$.
    As such, it can be defined either by its extreme points or through linear constraints \cite{Zaffalon_2002}. In the following (with a slight abuse of notation), we write $\theta^K\in K(X)$ to refer to the element $P(X \mid \theta^K)\in K(X)$.
    
    A \textit{conditional} credal set $K(X\mid Y)$ is usually defined elementwise, that is, obtained from $K(X,Y)$ by applying the Bayes rule to each of its elements \cite{Cozman_2004} (and possibly taking the convex hull).
    A collection of conditional credal sets $K(X\mid Y)$ is said to be \textit{separately specified} when $K(X\mid Y = y_1)$ and $K(X\mid Y = y_2)$ are unrelated for $y_1\neq y_2$, i.e., not constrained by each other.

    \begin{definition}[Locally specified CNs]
        \label{def:ls_cn}
        A \textit{locally specified CN} consists of a DAG $\mathcal{G}$ where each node $X$ is associated with a local credal set $K(X\mid \Pi_X)$.
    \end{definition}
    
    As is common in the literature, we consider $K(X\mid \Pi_X)$ to be separately specified. It becomes now visible how a locally and separately specified CN is akin to a BN. For any point parameter describing $P(X\mid \Pi_X = \pi_X)$ in the BN, there is now a closed set $K(X\mid\Pi_X = \pi_X)$ in the CN. What is missing, however, is a way to \textit{combine} the local credal sets, so that a global-to-local relation as the one in Eq.~\eqref{eq:bn_factorization} holds.
    While there is no single way to do this, a practical choice is offered by the CN \textit{strong extension} \cite{Cozman_2004}.

    \begin{definition}[Strong extension]
        The \textit{strong extension} of a locally and separately specified CN is a pair $(\mathcal{G}, K(\mathbf{X}))$, where $K(\mathbf{X})$ is the convex hull of:
        \begin{equation}
            \left\lbrace \prod_{X\in\mathbf{X}}P(X\mid \Pi_X, \theta_X) : P(X\mid \pi_X, \theta_X)\in K(X\mid \pi_X)\right\rbrace\, .
        \end{equation}
    \end{definition}
    
    In a strong extension, each distribution within $K(\mathbf{X})$ factorizes accordingly to $\mathcal{G}$ and Eq.~\eqref{eq:bn_factorization}. As such, the extension mimics the same conditional independence properties of a standard BN for all its extreme distributions \cite{Cozman_2000}. 

    \begin{remark}
        The following discussion introduces two approaches to defining the local sets $K(X\mid\Pi_X)$. The resulting CNs are both locally and separately specified, and we consider their strong extensions $(\mathcal{G}, K(\mathbf{X}))$ throughout the paper. Other types of CNs are out of the scope of the present article and can be found in Corani and de Campos \cite{Corani_2010}. Our main results hold regardless of how the CN was obtained, as long as they are locally and separately specified.
    \end{remark}
    
    The first approach to defining $K(X\mid\Pi_X)$ is based on the \textit{imprecise Dirichlet model} (IDM) \cite{Walley_1996}. The second approach consists of perturbing a given BN.
    Let $X$ be a random variable, $k=|\text{supp}(X)|$, and assume we want to estimate the parameters $\theta_X = \{\theta_1, \hdots, \theta_k\}$ where $\theta_i = p(x_i\mid \theta_X)$.
    In the \textit{precise} context of the Bayesian posterior estimator, the parameters are assumed to follow a Dirichlet distribution $Dir(\mathbf{c})$. Here, $\mathbf{c}=\{c_1,\hdots, c_k\}$ are constrained to $c_i>0~~\forall i$ and $\sum_i c_i = S$, where $S$ is a positive constant \cite{Walley_1996}. 
    Then, $\theta_{i}$ can be estimated from data as:
    \begin{equation}
        \theta_{i} = \frac{n_{i}+c_i}{N+S}\, ,
    \end{equation}
    where $n_i$ is the count of data items for which $X=x_i$, and $N$ is the data sample size. 
    Consider now an \textit{uncertain} version of the Bayesian posterior estimator. The \textit{local} application of IDM \cite{Corani_2010} allows each $c_i$ to vary within the interval $[0, S]$, leading to a set of priors for each $\theta_i$. The maximum a posteriori estimate for $\theta_i$ becomes now the interval:
    \begin{equation}
        p(x_i) = \left[\frac{n_{i}}{N+S}, \frac{n_{i}+S}{N+S}\right]\, .
    \end{equation}
    Parameters of the conditional distributions can be derived similarly:
    \begin{equation}
        p(x_i \mid \pi_j) = \left[\frac{n_{ij}}{n_j+S}, \frac{n_{ij}+S}{n_j+S}\right]\,,
    \end{equation}
    where $n_{ij}$ is the count of data items for which $(X=x_i, \Pi_X=\pi_j)$.

    Another approach to define $K(X\mid\Pi_X)$ is known as $\varepsilon$\textit{-contamination} \cite{Cozman_2000}, and consists in injecting uncertainty into a given BN. The BN may have been learned from data or provided by experts. \Cref{def:eps-cont} describes the concept for marginal distributions, but it extends naturally to each local conditional probability.
    
    \begin{definition}[$\varepsilon$-contamination]
        \label{def:eps-cont}
        Let $X$ be a discrete random variable with probability mass function $p(\cdot\mid \theta_X)$. Given $\varepsilon\in(0,1)$, the $\varepsilon$-contamination of $p(\cdot\mid \theta_X)$ is defined as the set of all probability mass functions $r(\cdot)$ so that:
        \begin{equation}
            \label{eq:econtamin}
            r(x) \geq (1-\varepsilon)p(x\mid \theta_X), \quad \forall x\in\textnormal{supp}(X)\, . 
        \end{equation}
    \end{definition}

\section{Problem Statement}
\label{sec:prob_stat}

    Consider a \textit{general population} $\mathcal{P}$ of individuals, where $\mathbf{X}$ represents the measured variables for each of them, and let $(\mathcal{G}, \theta_0)$ be the unknown BN that generated $\mathcal{P}$.
    Let also $(\mathcal{G}, \widehat{\theta}_{\mathcal{T}})$ be the BN learned from a subsample of individuals $\mathcal{T} \subset \mathcal{P}$, which we call the \textit{target population}. 
    Data in $\mathcal{T}$ are \textit{private}, meaning that no one has direct access to them, apart from the data owner and the learning algorithm. However, the BN may be released for research purposes. 
    Let the attacker have access to the released BN and auxiliary data $\mathcal{R} \subset \mathcal{P}$, referred to as the \textit{reference population}. In this setting, a \textit{tracing attack} \cite{Carlini_2022}, also known as \textit{membership inference attack}, consists of detecting whether a specific individual $\mathbf{x}\in\mathcal{P}$ belongs to $\mathcal{T}$.
    In a medical field, $\mathcal{P}$ could represent all patients suffering from a particular disease, $\mathcal{T}$ a hospital database that handles some patients belonging to $\mathcal{P}$, and $\mathcal{R}$ a publicly available clinical registry. 
    Another example involves a competing insurance company that aims to attract more clients from a rival firm. In this context, $\mathcal{P}$ represents a vulnerable category of individuals, while $\mathcal{T}$ and $\mathcal{R}$ denote subsets of people associated with the competing firm and the insurance company, respectively.
    
    The attacker information, namely  $(\mathcal{G}, \theta_0)$ and $\mathcal{R}$, is exploited to build a \textit{decision rule} $\phi(\cdot)$ so that, given $\mathbf{x} \in\mathcal {P}$, it yields either $\phi(\mathbf{x})=1$, hence concluding $\mathbf{x}\in\mathcal{T}$, or $\phi(\mathbf{x})=0$, concluding $\mathbf{x}\notin\mathcal{T}$.
    A possible choice for the rule $\phi(\cdot)$ involves conducting the \textit{log-likelihood ratio} (LLR) test \cite{Lehmann_2022}.
    In our case, the test is expressed by the following competing hypotheses. We assign $\phi(\mathbf{x})=1$ whenever we reject $H_0$ in favor of $H_1$, and $\phi(\mathbf{x})=0$ otherwise.
    
    \begin{equation}
        \label{eq:test}
        \begin{cases}
            H_0: \mathbf{x} \notin\mathcal{T} \\
            H_1: \mathbf{x} \in \mathcal{T}\, .
        \end{cases}
    \end{equation}    
    Surely, this test has only limited qualities in exposing elements of $\mathcal{T}$, but these are typically considered too much to be allowed \cite{Dwork2017}.
    
    The definition of the test statistic depends on the values of $\widehat{\theta}_{\mathcal{T}}$ and $\theta_0$ \cite{Lehmann_2022} (see Table~\ref{tab:notation} for notation). However, $\theta_0$ is the vector of \textit{true} BN parameters, which are unknown. When $\mathcal{R}$ is large \textit{enough}, the attacker can assume that the MLE estimate $\widehat{\theta}_{\mathcal{R}}$ is a reasonable approximation of $\theta_0$. Thus, the LLR statistic is given by:
    \begin{equation}
        \label{eq:L}
        L: \text{supp}(\textbf{X})  \rightarrow \mathbb{R}, \quad
        L(\mathbf{x}) = \log \left[ \frac{P(\mathbf{x} \mid \widehat{\theta}_{\mathcal{T}})}{P(\mathbf{x} \mid \widehat{\theta}_{\mathcal{R}})} \right]\, .
    \end{equation}

    The test rejects $H_0$ whenever the LLR statistic is \textit{large enough}. 
    Specifically, the test operates as follows.
    Let $\alpha$ be the test Type I error, namely $\alpha = P(\phi(\mathbf{x})=1 \mid \mathbf{x} \notin\mathcal{T})$ represents the error incurred by the attacker when concluding that $\mathbf{x}\in\mathcal{T}$ when the opposite holds instead.
    The attacker first sets an \textit{error level} $\alpha$ that can be tolerated, and estimates the distribution of $L(\cdot)$ under the null hypothesis $H_0$.
    Next, a threshold $\tau(\alpha)$ is chosen, so that $P(L(\mathbf{x})>\tau(\alpha)\mid \mathbf{x}\notin\mathcal{T})=\alpha$.\footnote{If $F(\cdot)$ is the cumulative distribution function (CDF) of $\Lambda(\cdot)$ under $H_0$, then $\tau(\alpha)= F^{-1}(1-\alpha)$.}
    Finally, the null hypothesis $H_0: \mathbf{x}\notin\mathcal{T}$ is rejected in favor of $H_1$ whenever $L(\mathbf{x})>\tau(\alpha)$. Otherwise, if $L(\mathbf{x})\leq\tau(\alpha)$, there is insufficient evidence to conclude that $\mathbf{x}\in\mathcal{T}$.
    The \textit{attack power} is defined as $\beta = P(\phi(\mathbf{x})=1 \mid \mathbf{x} \in \mathcal{T})$. It has been proved that the LLR test achieves the highest power among all other tests, for any given error level $\alpha$ \cite{Lehmann_2022}. 

    Murakonda et al. \cite{Murakonda_2021} proved a relationship between $\alpha$ and $\beta$ that is independent of the parameters $\widehat{\theta}_{\mathcal{T}}$. Instead, this relationship can be deduced from the graphical structure and the sample size of the target population $\mathcal{T}$.

    \begin{theorem}[\cite{Murakonda_2021}]
        \label{th:bound}
        For any error level $\alpha$, it holds:
        \begin{equation}
            z_{\alpha} + z_{1-\beta} \approx \sqrt{\frac{C(\mathcal{G})}{|\mathcal{T}|}}\, ,
        \end{equation}
        where $z_{s}$, $0 < s <1$, is the quantile at level $1-s$ of the Standard Normal distribution $\mathcal{N}(0,1)$.
    \end{theorem}    

\section{Theoretical Analysis}
\label{sec:theory}

    In this section, we adapt the notion of tracing attacks to CNs. Then, we address and provide answers to the following research questions: 
    \begin{enumerate}[label={\bf RQ\arabic*}, leftmargin=*, align=left]
        \item Does releasing a CN lead to a lower tracing attack power than releasing a BN?
        \item Are there some conditions under which the attacker can infer the true BN by observing the released CN?
        \item Does the attack power depend on the CN uncertainty?
    \end{enumerate}

\subsection{Attacking Credal Networks}
    \label{sec:attack_CN}
    When a BN is released, the attacker has access to $(\mathcal{G}, \widehat{\theta}_{\mathcal{T}})$ and also to the reference population $\mathcal{R}$. Therefore, the MLE $\widehat{\theta}_{\mathcal{R}}$ can be obtained from $\mathcal{R}$ and the knowledge about the released DAG $\mathcal{G}$.
    On the contrary, when a CN is released, the attacker only has access to a \textit{neighborhood} of $\widehat{\theta}_{\mathcal{T}}$, namely the credal set $\widehat{K}_{\mathcal{T}}(\mathbf{X}) \subset \Theta$ estimated from $\mathcal{T}$. Therefore, a new \textit{uncertain} version of the tracing attack must be designed. 
    
    While the test hypotheses are the same as in Eq.~\eqref{eq:test}, the LLR statistic now depends on the choice of $\theta^K\in \widehat{K}_{\mathcal{T}}(\mathbf{X})$. By using the properties of logarithms, we rewrite Eq.~\eqref{eq:L} and adapt it to the \textit{uncertain case} as follows:
    \begin{align}
        & L: \text{supp}(\textbf{X}) \times \widehat{K}_{\mathcal{T}}(\mathbf{X}) \rightarrow \mathbb{R},  \\
        & L(\mathbf{x}, \theta^K) = \log[P(\mathbf{x} \mid \theta^K)] - \log[P(\mathbf{x} \mid \widehat{\theta}_{\mathcal{R}})]     \label{eq:llr_theta} \,.
    \end{align}
    The decision rules related to the uncertain version of the test are then $\phi(\cdot, \theta^K)$, one for each $\theta^K\in\widehat{K}_{\mathcal{T}}(\mathbf{X})$. 
    By setting $\theta^K=\widehat{\theta}_{\mathcal{T}}$ in Eq.~\eqref{eq:llr_theta}, we obtain the \textit{precise} version of $L(\cdot)$ and $\phi(\cdot)$ as described in \Cref{sec:prob_stat}.
    
    From the attacker perspective, a point $\theta^K\in\widehat{K}_{\mathcal{T}}(\mathbf{X})$ has to be selected for $\phi(\cdot, \cdot)$ so that, whenever $\phi(\mathbf{x}, \theta^K)=1$ then $\phi(\mathbf{x}, \widehat{\theta}_{\mathcal{T}})=1$. If this is the case, we say the test defined by $\phi(\cdot, \theta^K)$ is \textit{consistent} with $\phi(\cdot, \widehat{\theta}_{\mathcal{T}})$.
    If a random point is chosen within $\widehat{K}_{\mathcal{T}}(\mathbf{X})$, it may coincide with $\widehat{\theta}_{\mathcal{T}}$ by chance. However, any other point would represent the MLE estimate on some other subset $\mathcal{H}\subset \mathcal{P}$, where $\mathcal{H}$ is unknown. In the worst case, $\mathcal{H}\cap\mathcal{T}=\varnothing$.
    \Cref{def:tracing_CN} provides the notion of a tracing attack against CNs.

    \begin{definition}[Tracing attack against CNs]
        \label{def:tracing_CN}
        Let $\mathcal{T}, \mathcal{R} \subset \mathcal{P}$ the target and reference population respectively, and $(\mathcal{G}, \widehat{K}_{\mathcal{T}}(\mathbf{X}))$ the CN estimated from $\mathcal{T}$.
        Let also $\widehat{\theta}^K_{\mathcal{R}}\in\widehat{K}_{\mathcal{T}}(\mathbf{X})$ be the MLE estimate on $\mathcal{R}$, namely:
        \begin{equation}
            \widehat{\theta}^K_{\mathcal{R}} = \argmax_{\theta^K \in \widehat{K}_{\mathcal{T}}(\mathbf{X})} \sum_{\mathbf{x}\in\mathcal{R}}P(\mathbf{x}\mid \theta^K)\, .
        \end{equation}
        We define the \textit{tracing attack} against $(\mathcal{G}, \widehat{K}_{\mathcal{T}}(\mathbf{X}))$ as the attack given by $L(\cdot, \widehat{\theta}^K_{\mathcal{R}})$ and $\phi(\cdot, \widehat{\theta}^K_{\mathcal{R}})$.
    \end{definition}
    
    \Cref{th:consistency} proves the test's consistency in expectation. Lemmas \ref{le:L} and \ref{le:tau} lead it.

    \begin{remark}
        \label{re:asymp}
        The following results hold in the sample size limit, leveraging the asymptotic distribution of the log-likelihood function and the LLR statistic \cite{Lehmann_2022}.
    \end{remark}

    \begin{lemma}
        \label{le:L}
        $\mathbb{E}\left[L(\mathbf{x}, \widehat{\theta}^K_{\mathcal{R}})\mid\mathbf{x}\in\mathcal{T}\right] \leq \mathbb{E}\left[L(\mathbf{x}, \widehat{\theta}_{\mathcal{T}})\mid\mathbf{x}\in\mathcal{T}\right]$\, .
    \end{lemma}

    \begin{proof}
        The second term of $L(\mathbf{x}, \theta^K)$, namely $\log P(\mathbf{x} \mid \widehat{\theta}_{\mathcal{R}})$, is constant for any $\theta^K\in\widehat{K}_{\mathcal{T}}(\mathbf{X})$. Recalling $\widehat{\theta}_{\mathcal{T}}$ is the MLE estimate over $\mathcal{T}$, we have:
        \begin{equation}
            \label{eq:mle_l1}
            \sum_{\mathbf{x}\in\mathcal{T}}L(\mathbf{x}, \theta^K) \leq \sum_{\mathbf{x}\in\mathcal{T}}L(\mathbf{x}, \widehat{\theta}_{\mathcal{T}})\, , \quad \forall~\theta^K\in \widehat{K}_{\mathcal{T}}(\mathbf{X})\, .
        \end{equation}
        The conclusion follows by exploiting \Cref{re:asymp}.
    \end{proof}
    Notice that the inequality in Eq.~\eqref{eq:mle_l1} holds for any $\theta\in\Theta$ in the left-hand side term. 
    
    \begin{lemma}
        \label{le:tau}
        For any $\alpha\in[0,1]$, it holds $\tau(\alpha, \widehat{\theta}_{\mathcal{T}}) \leq \tau(\alpha, \widehat{\theta}^K_{\mathcal{R}})$. 
    \end{lemma}

    \begin{proof}
        Recalling $\widehat{\theta}^K_{\mathcal{R}}$ is the MLE estimate over $\mathcal{R}$ constrained to $\widehat{K}_{\mathcal{T}}(\mathbf{X})$, we obtain:
        \begin{equation}
            \label{eq:l_r}
            \sum_{\mathbf{x}\in\mathcal{R}}L(\mathbf{x}, \widehat{\theta}_{\mathcal{T}}) \leq \sum_{\mathbf{x}\in\mathcal{R}}L(\mathbf{x}, \widehat{\theta}^K_{\mathcal{R}})\leq\sum_{\mathbf{x}\in\mathcal{R}}L(\mathbf{x}, \widehat{\theta}_{\mathcal{R}})=0\, ,
        \end{equation}
        where the first inequality exploits the fact that $\widehat{\theta}_{\mathcal{T}}=\widehat{\theta}^K_{\mathcal{T}}$.
        Assume two tracing attacks are conducted: one against the BN and the other against the CN, both with the same error level $\alpha$.
        The former uses $L(\cdot, \widehat{\theta}_{\mathcal{T}})$ as the test statistic, while the latter uses $L(\cdot, \widehat{\theta}^K_{\mathcal{R}})$. Under the null hypothesis $H_0: x\notin \mathcal{T}$, the related thresholds are so that:
        \begin{gather}
            P(\log P(\mathbf{x} \mid \widehat{\theta}_{\mathcal{T}}) > \tau(\alpha, \widehat{\theta}_{\mathcal{T}})+l) =\alpha\, ,\label{eq:alpha1} \\
            P(\log P(\mathbf{x} \mid \widehat{\theta}^K_{\mathcal{R}}) > \tau(\alpha, \widehat{\theta}^K_{\mathcal{R}})+l)  =\alpha\,, \label{eq:alpha2}
        \end{gather}
        where $l=\log P(\mathbf{x} \mid \widehat{\theta}_{\mathcal{R}})$ is a constant. Due to Eq.~\eqref{eq:l_r} and \Cref{re:asymp}, and following normality results of the LLR statistic (see~\cite{Murakonda_2021}), the two left-hand side terms of Eq.~\eqref{eq:alpha1} and Eq.~\eqref{eq:alpha2} are nested iff $\tau(\alpha, \widehat{\theta}_{\mathcal{T}}) \leq \tau(\alpha, \widehat{\theta}^K_{\mathcal{R}})$.
    \end{proof}

    From \Cref{le:L} and \Cref{le:tau} we obtain the consistency (in expectation) of \Cref{def:tracing_CN}, summarized in \Cref{th:consistency}. Practically speaking, whenever the attacker rejects the null hypothesis $H_0$ and concludes that $\mathbf{x} \in \mathcal{T}$ by using the released CN, the same conclusion would likely be reached if the attacker had instead utilized the BN. Note that this holds regardless of how the CN was learned.
    
    \begin{theorem}
        \label{th:consistency}
        In expectation over $\mathbf{x}\in\mathcal{T}$, it holds: 
        \begin{equation}
            L(\mathbf{x}, \widehat{\theta}^K_{\mathcal{R}}) \geq \tau(\alpha, \widehat{\theta}^K_{\mathcal{R}}) \Rightarrow L(\mathbf{x}, \widehat{\theta}_{\mathcal{T}}) \geq \tau(\alpha, \widehat{\theta}_{\mathcal{T}})\, .
        \end{equation}
    \end{theorem}

    If the attacker picks randomly $\theta^K \in \widehat{K}_{\mathcal{T}}(\mathbf{X})$, the consistency of the test (\Cref{def:tracing_CN}) could not be proved. In that case, the inequalities in Eq.~\eqref{eq:l_r} may not hold.

\subsection{Privacy Guarantees}
    \label{sec:pr_g}

    This section investigates how privacy varies when releasing a CN instead of a BN. In particular, the following theorem answers the research question \textbf{RQ1}.

    \begin{theorem}
        \label{th:power}
        For any given value of the Type I error $\alpha$, the power of the attack against the CN is no higher than that achieved against the BN. Specifically, for any $\alpha\in[0,1]$, it holds $\beta(\widehat{\theta}^K_{\mathcal{R}}) \leq \beta(\widehat{\theta}_{\mathcal{T}})$.
    \end{theorem}

    \begin{proof}
        From the definition of the test power, it holds the following:
        \begin{align}
            \beta(\widehat{\theta}^K_{\mathcal{R}}) &=
            P(\log P(\mathbf{x}\mid \widehat{\theta}^K_{\mathcal{R}}) >\tau(\alpha, \widehat{\theta}^K_{\mathcal{R}})+l\mid \mathbf{x}\in\mathcal{T}) \\
            &=  \frac{1}{|\mathcal{T}|}\sum_{\mathbf{x}\in \mathcal{T}} \mathbb{I}(\log P(\mathbf{x}\mid \widehat{\theta}^K_{\mathcal{R}}) >\tau(\alpha, \widehat{\theta}^K_{\mathcal{R}})+l)  \\
            &\leq  \frac{1}{|\mathcal{T}|}\sum_{\mathbf{x}\in \mathcal{T}} \mathbb{I}(\log P(\mathbf{x}\mid \widehat{\theta}_{\mathcal{T}}) > \tau(\alpha, \widehat{\theta}^K_{\mathcal{R}})+l)  \\
            &\leq  \frac{1}{|\mathcal{T}|}\sum_{\mathbf{x}\in \mathcal{T}} \mathbb{I}(\log P(\mathbf{x}\mid \widehat{\theta}_{\mathcal{T}}) > \tau(\alpha, \widehat{\theta}_{\mathcal{T}})+l) \\
            &=  P(\log P(\mathbf{x}\mid \widehat{\theta}_{\mathcal{T}}) >\tau(\alpha, \widehat{\theta}_{\mathcal{T}})+l\mid \mathbf{x}\in\mathcal{T}) \\ 
            &= \beta(\widehat{\theta}_{\mathcal{T}}) \,,
        \end{align}
        where $\mathbb{I}(\cdot)\in\{0,1\}$ is the \textit{indicator function} of its argument, and the inequalities exploit \Cref{re:asymp}.
    \end{proof} 

    For the final part of this section, assume the attacker knows how the released CN was learned, for instance, whether the local application of IDM or the $\varepsilon$-contamination was employed. In this scenario, the privacy gained by using CNs may be compromised, allowing $\widehat{\theta}_{\mathcal{T}}$ to be derived. 
    Specifically, Lemmas \ref{le:idm} and \ref{le:eps} address the research question \textbf{RQ2}.

    \begin{lemma}
        \label{le:idm}
        Consider a CN learned by the local application of IDM. The true BN can be recovered whenever the attacker knows $S$.
    \end{lemma}

    Each parameter interval is of width $\frac{S}{N+S}$, whose value can be obtained by the difference of the interval extremes.
    Not only can $\widehat{\theta}_{\mathcal{T}}$ be obtained, but the sample size $|\mathcal{T}|$ can also be inferred, providing the attacker with absolute frequencies. 
    In the light of \Cref{le:idm}, knowledge about $S$ and $N$ must be kept secret while releasing the CN learned with IDM. Different $S_X$ could be used as an additional protection technique, e.g., one for each random variable $X\in\mathbf{X}$. 

    \begin{lemma}
        \label{le:eps}
        Consider an $\varepsilon$-contaminated CN defined in \Cref{sec:cn}. Then, the attacker can recover the underlying BN without additional information.
    \end{lemma}

    By exploiting Eq.~\eqref{eq:econtamin}, each parameter interval is uniquely defined by the original parameter value and $\varepsilon$, which is the interval width.
    Trade-offs between the uncertainty included in the network and the privacy to be preserved depend on the specific case. Intuitively, larger parameter intervals lead to greater privacy, but also result in a decrease in the model's utility. Depending on $S$, or $\varepsilon$, we can mark two extreme situations:

    \begin{enumerate}
        \item $\widehat{K}_{\mathcal{T}}(\mathbf{X}) \equiv \{\widehat{\theta}_{\mathcal{T}}\}$. In this case, we release a single BN, and the power of a tracing attack is approximated by \Cref{th:bound} \cite{Murakonda_2021}. Hence, we release an unmasked model, thereby obtaining the minimum privacy.
        \item $\widehat{K}_{\mathcal{T}}(\mathbf{X}) \equiv \Theta$. In this case, each parameter interval is $[0,1]$, thus only the DAG is released. Under \Cref{def:tracing_CN}, the LLR statistic is constantly zero. Hence, we release the least information while ensuring the highest level of privacy against tracing attacks.\footnote{The DAG itself may still reveal sensitive information \cite{Wang_2020}.}
    \end{enumerate}

\section{Experiments}
    \label{sec:exp}
    We conduct experiments to answer research question \textbf{RQ3}.
    In particular, we aim to empirically estimate how enhanced privacy, i.e., the power $\beta$ of the test, varies based on: (i) the complexity of the DAG, and (ii) the uncertainty introduced in the CN. This section investigates the local application of IDM only, and not the $\varepsilon$-contaminated class of CNs, due to the pointed concerns regarding privacy breaches (\Cref{sec:pr_g}). Hence, the uncertainty is represented by a global constant $S$ (\Cref{sec:cn}), and the attacker does not necessarily know how the CN was learned. 
    We note that, even for large networks, the experiments remain computationally feasible. We make our source code available for reproducibility and research purposes.\footnote{\href{https://github.com/Niccolo-Rocchi/BN-Privacy}{https://github.com/Niccolo-Rocchi/BN-Privacy}}
    
\subsection{Experimental Setting}

    We consider all DAGs derived from the combinations of: number of nodes $M\in\{10, 20, 30, 50\}$, edge density $E_d\in\{2, 3, 4\}$. The latter denotes the average number of ingoing edges per node; thus, the total number of edges in each network is $E=M\cdot E_d$. The consideration of large networks stems from practical frameworks, such as omics data modeling \cite{Suter_2022}.
    The following presents the experimental pipeline for a given DAG $\mathcal{G}$ with $M$ nodes and $E$ edges.
    \begin{figure*}[h!]
        \centering        \includegraphics[width=\linewidth]{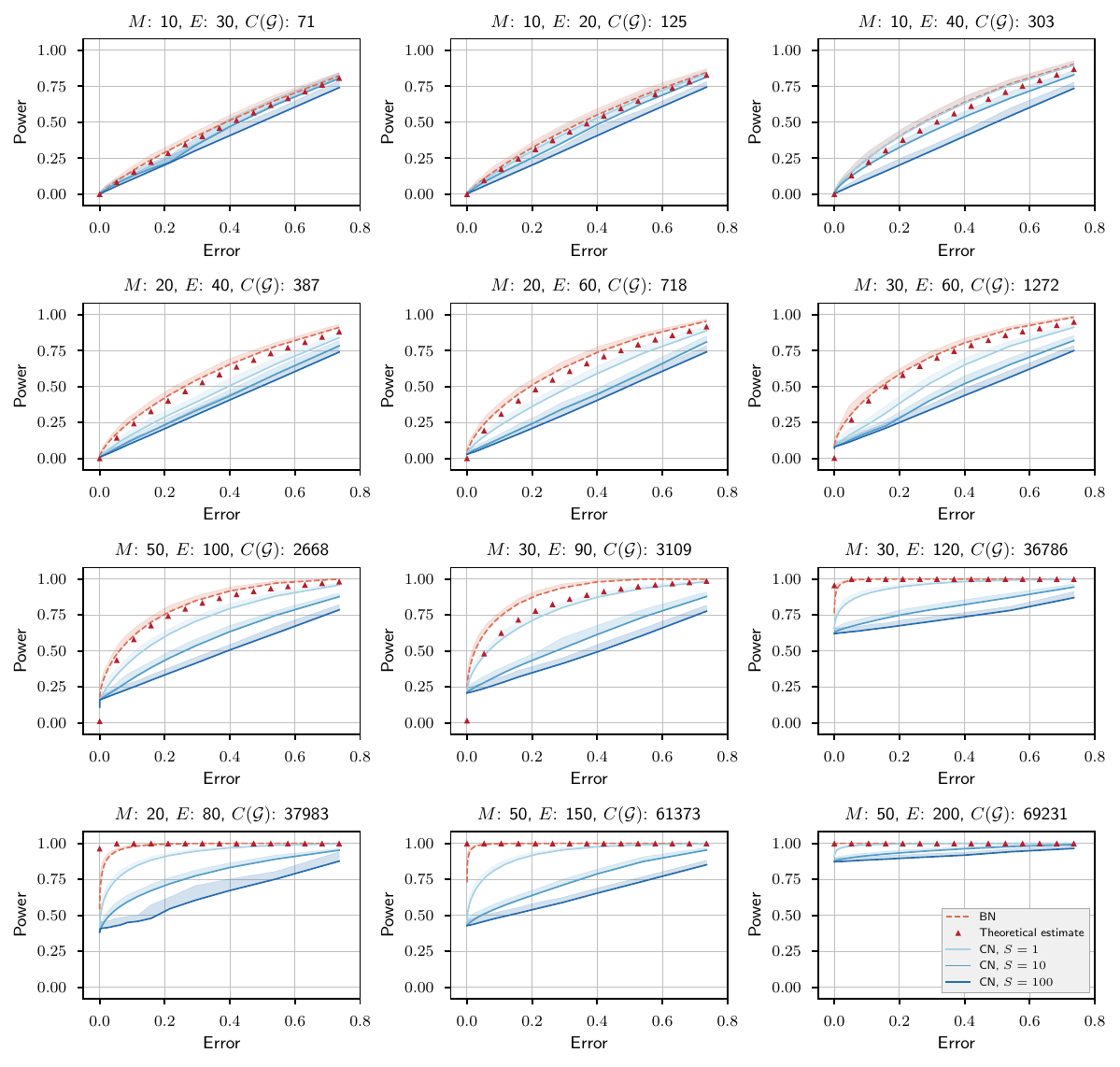}
        \caption{Error vs. power rates for experimental tracing attacks across various node (rows) and edge (columns) configurations. For each setup, $\mathcal{T}$ and $\mathcal{R}$ are sampled \numprint{200} times from $\mathcal{P}$. Lines (solid or dashed) represent average power, while shaded areas indicate the maximum power achieved across those samples.}
        \label{fig:results}
    \end{figure*}

\paragraph{Data generation}
    A BN $(\mathcal{G}, \theta_0)$ over discrete variables (maximum 3 categories each) is randomly generated and taken as the ground truth for the data generation process. 
    Then, we sample \numprint{5000} unique individuals from $(\mathcal{G}, \theta_0)$ and consider them as the general population $\mathcal{P}$.    
    The reference and target population $\mathcal{R}, \mathcal{T}$ are randomly sampled from $\mathcal{P}$ with sample size ratios of $0.5$ and $0.25$, respectively, ensuring no overlap. 
    Subsampling from $\mathcal{P}$ is performed \numprint{200} times to ensure variability in the sampling. 
\paragraph{Tracing attack against BN}
    Given a sequence of Type I errors $\alpha_1, \hdots,\alpha_k \in [10e^{-4}, 0.631]$, $k=30$, the aim is to compute the corresponding power of a tracing attack, i.e, $\beta_1, \hdots,\beta_k$.
    First, for each $\alpha_i$, we compute the theoretical estimate of $\beta_i$ as given by \Cref{th:bound}, by letting $C(\mathcal{G})$ be the complexity of $\mathcal{G}$ and $|\mathcal{T}|$ the target sample size.
    To conduct the attack against the BN, we perform the MLE estimation of $(\mathcal{G}, \widehat{\theta}_{\mathcal{R}})$ and $(\mathcal{G}, \widehat{\theta}_{\mathcal{T}})$ from $\mathcal{R}$ and $\mathcal{T}$, respectively. Recall that $\widehat{\theta}_{\mathcal{T}}$ is released by the modeler and is sensitive information, while $\widehat{\theta}_{\mathcal{R}}$ is always accessible to the attacker.
    Then, for each $\alpha_i$, we perform the tracing attack by using the statistics $L(\cdot, \widehat{\theta}_{\mathcal{T}})$. In particular, for a fixed $\alpha_i$, the test considers any data point $\mathbf{x}\in\mathcal{P} \setminus \mathcal{R}$ and assign either $\mathbf{x}\notin\mathcal{T}$ or $\mathbf{x}\in\mathcal{T}$. The true positive rate $\beta_i$ is computed for any error level $\alpha_i$.

\paragraph{Tracing attack against CN}
    The tracing attack against the CN is performed in a manner similar to that described in the previous paragraph. The only difference consists in estimating $\widehat{K}_{\mathcal{T}}(\mathbf{X})$, released by the modeler instead of $\widehat{\theta}_{\mathcal{T}}$. The credal set is estimated from the target population for a fixed $S$.
    The attacker, having access to $\widehat{K}_{\mathcal{T}}(\mathbf{X})$, estimates $\widehat{\theta}^K_{\mathcal{T}}$ by optimization techniques.
    Then, for each error level $\alpha_i$, the tracing attack is performed on $\mathcal{P} \setminus \mathcal{R}$ by employing the statistics $L(\cdot, \widehat{\theta}^K_{\mathcal{R}})$, and the true positive rate $\beta_i$ is computed.
    
\subsection{Results \& Discussion}

    Results are shown in \Cref{fig:results}.
    We report results for CNs learned by the local IDM with $S\in\{1,10,100\}$. We also considered $S=1000$, however, it gave results compatible with $S=100$, hence it is not reported here.

    From the attacker's perspective, the \textit{optimal curve} would be L-shaped towards the top-left corner; in other words, the attack that achieves the maximum power value $\beta=1$ when the error level is set to its minimum value $\alpha=0$.
    From the privacy-preserving perspective, instead, a graphical model ensures complete privacy when its curve is L-shaped towards the bottom-right corner of the plot; in other words, a power $\beta=0$ for any Type I error $\alpha<1$ and a power $\beta=1$ only for $\alpha=1$. 
    
    Our results demonstrate that CNs are not less private than BNs, as the CN curves are consistently below, or coincide with, the BN curves.
    This statement is further supported by the observation that the maximum power achieved with CNs does not intersect the BN curves at higher DAG complexities and lower errors.
    Hence, CNs are more private than BNs in most experiments. The area between each pair of BN-CN curves indicates the privacy gained by exposing the imprecise model instead of the precise one. We observe that this area does not necessarily increase with network complexity. However, it increases with the uncertainty injected into the CN, allowing us to modulate privacy by tuning the $S$ hyperparameter.
    
    As noted above, choosing $S=1000$ does not yield privacy benefits relative to $S=100$, albeit to a small extent in more complex graphs.
    This suggests that the privacy gained in CNs may be upper-bounded, or, more generally, dependent on both the sample size and the model complexity. 
    Nevertheless, the benefits from even smaller $S$ are visible. As a consequence, by introducing a small amount of uncertainty into the CN, we can still achieve precise inferences (indicating high model's utility) while ensuring poorer attack power (indicating high privacy). Thus, one could regard a CN learned via the local IDM with $S=1$ or $S=10$ as already a reasonable balance between privacy and the model's utility, although further research is warranted.

    Finally, we observe the BN curves adhere to the theoretical estimate in most experiments; however, this estimate might not be considered as \textit{upper bound} to the loss of information in BNs, as hinted in Murakonda et al. \cite{Murakonda_2021}.

\section{Conclusions}
    This work introduces a promising and ready-to-use approach to the issue of privacy-aware parameter learning in BNs.
    In particular, we adopt the imprecise probability theory to obfuscate the BNs' parameters before their release; the degree of introduced uncertainty effectively hinders tracing attacks despite not neutralizing them completely. Contrary to a noisy BN, the released CN retains significance regarding the actual BN parameters. It thus keeps its utility for further inferences, as one can use the CN to issue provably correct inferential bounds, which are impossible after adding noise to a BN.

    Our framework is directly relevant to healthcare applications, where data privacy is paramount and data sets are often small and fragmented across institutions.
    It also aligns well with federated learning scenarios, where individual data sources, such as hospitals, collaboratively train AI models. While most of the existing federated learning research focuses on the structure learning task, little attention has been given to the parameter learning task, especially under population heterogeneity. We hope this work can begin a research direction to close that gap.

\section{Future Work}
    
    We believe we can achieve sharper results than \Cref{th:consistency} and stronger guarantees than \Cref{th:power} by investigating the non-asymptotic behavior of the log-likelihood and LLR functions; the experimental results also support this argument. Nevertheless, the paper's insights are already applicable to practice and pave the way for a new line of research involving CNs and privacy.
    Moreover, we aim to identify other sensitive information than that provided in \Cref{sec:pr_g}, to prevent the reconstruction of the masked BN. 
    Motivated by our findings, a fruitful research direction involves providing a theoretical approximation of the attack effectiveness against CNs, similar to the one discussed in Murakonda et al. \cite{Murakonda_2021}.
    
    On the empirical side, we recognize the need for more ablation studies to evaluate the sensitivity of results across different networks, sample sizes, and uncertainty levels. 
    Moreover, while CNs achieve higher utility than noisy BNs in providing correct inferential bounds, this is primarily a qualitative property. An empirical evaluation of the accuracy and calibration of both models is an ongoing work that we are addressing.
    Future work should also investigate the applicability of our framework to alternative classes of CNs, as well as different attack techniques.
    Finally, we will evaluate and fine-tune our methodology using real-world data as soon as suitable and relevant scenarios become available.    

\section*{Acknowledgments}
    This work was supported by the MUR under the grant ``Dipartimenti di Eccellenza 2023-2027'' of the Department of Informatics, Systems and Communication of the University of Milano-Bicocca, Milan, Italy, and by the National Plan for NRRP Complementary Investments (Project n. PNC0000003 - AdvaNced Technologies for Human-centrEd Medicine (ANTHEM)).
    Niccolò Rocchi is funded by NextGeneration EU and Fondazione IRCCS Istituto Nazionale dei Tumori, Milan, Italy, in the Horizon Europe project IDEA4RC framework.
    Cassio de Campos thanks the support from the Eindhoven Artificial Intelligence Systems Institute, EU European Defence Fund via project KOIOS (EDF-2021-DIGIT-R-FL-KOIOS), and the Dutch Research Council (NWO) via project NGF.1609.242.024.

\bibliographystyle{plain}
\bibliography{refs}

\end{document}